\documentclass{article}

\usepackage[english]{babel}

\usepackage[letterpaper,top=2cm,bottom=2cm,left=3cm,right=3cm,marginparwidth=1.75cm]{geometry}

\usepackage{amsmath}
\usepackage{graphicx}
\usepackage[utf8]{inputenc}
\DeclareUnicodeCharacter{1F60E}{\glassemoji}
\usepackage[colorlinks=true, allcolors=blue]{hyperref}
\usepackage{amssymb}
\usepackage{amsthm, thm-restate}

\newtheorem{lemma}{Lemma}
\newtheorem*{conjecture}{Conjecture}

\usepackage{natbib}

\title{Notes on Sampled Gaussian Mechanism}
\author{Nikita P. Kalinin}

\begin{document}
\maketitle

\begin{abstract}
   In these notes, we prove a recent conjecture posed in the paper by Räisä, O. et al. [Subsampling is not Magic: Why Large Batch Sizes Work for Differentially Private Stochastic Optimization (2024)]. Theorem 6.2 of the paper asserts that for the Sampled Gaussian Mechanism—a composition of subsampling and additive Gaussian noise—the effective noise level, $\sigma_{\text{eff}} = \frac{\sigma(q)}{q}$, decreases as a function of the subsampling rate $q$. Consequently, larger subsampling rates are preferred for better privacy-utility trade-offs. Our notes provide a rigorous proof of Conjecture 6.3, which was left unresolved in the original paper, thereby completing the proof of Theorem 6.2.

\end{abstract}

\section{Introduction}

Differential privacy (DP) has become a standard for ensuring privacy in machine learning models. One of the widely used techniques for achieving DP is the Sampled Gaussian Mechanism, which combines Gaussian noise addition with subsampling. The subsampling process, often implemented through Poisson subsampling, selects each data point independently with a probability $q$, referred to as the sampling rate. While increasing the subsampling rate $q$ generally reduces the subsampling variance, the impact on the Gaussian noise variance $\sigma^2$ is more complex. A higher subsampling rate diminishes the privacy amplification effect, requiring an increase in $\sigma$ to maintain the same level of privacy.  This introduces a nuanced trade-off between subsampling and noise addition that is not straightforward to analyze.  Theorem 6.2 in \cite{subsampling} suggests that the effective noise level, defined as $\sigma_{\text{eff}} = \frac{\sigma(q)}{q}$, decreases with higher $q$, favoring larger subsampling rates. However, this result depends on an unresolved conjecture, which we prove in these notes.

\subsection{Sampled Gaussian Mechanism}
Let $q \in (0, 1]$ be a sampling rate. Consider the composition of the Gaussian Mechanism with random subsampling. Then, for given $\epsilon > 0$ and $q\in(0,1]$, we have the following dependence between $\delta$ and $\sigma$:

\begin{equation}
    \delta(q) = q Pr\left(Z \ge \sigma(q) \log\left(\frac{h(q)}{q}\right) - \frac{1}{2\sigma(q)}\right) - h(q) Pr\left(Z \ge \sigma(q) \log\left(\frac{h(q)}{q}\right) + \frac{1}{2\sigma(q)}\right),
    \label{eq:delta_q}
\end{equation}
where $h(q):= e^{\epsilon} - 1 + q$ and $Z$ is a standard normal random variable. Let us denote:

\begin{equation}
    a:= \frac{1}{2\sqrt{2}\sigma(q)}; \;\;\; b:= \frac{\sigma(q)}{\sqrt{2}}\log\left(\frac{e^{\epsilon} - 1 + q}{q}\right).
\end{equation}

Then the conjecture from \citet{subsampling} has the following form:

\begin{conjecture}
For $\epsilon, q \ge 4\delta$, we have $a - b < 0$. 
\end{conjecture}

\section{Proof of Conjecture}

We denote the density function of the standard Gaussian distribution as $\phi(x)$, and the cumulative density function as $\Phi(x)$. 
Then we introduce an auxiliary function
\begin{equation}
    \Psi_{\epsilon, q}(\sigma) :=  q \Phi\left(-\sigma \log\left(\frac{h(q)}{q}\right) + \frac{1}{2\sigma}\right) - h(q)\Phi\left(-\sigma \log\left(\frac{h(q)}{q}\right) - \frac{1}{2\sigma}\right).
\end{equation}
Note that equality~\eqref{eq:delta_q} corresponds to the identity $\delta(q) = \Psi_{\epsilon, q}(\sigma(q))$.

First, we prove that $\Psi_{\epsilon, q}(\sigma)$ is a decreasing function of $\sigma$.

\begin{lemma}[Monotonicity of  $\Psi$]
For any $\epsilon > 0$ and $q\in(0,1]$, $\Psi_{\epsilon,q}(\sigma)$ is an invertable and strictly monotonically decreasing function with respect to $\sigma\in\mathbb{R}_+$.
\end{lemma}

\begin{proof}
It suffices to show that 
\begin{equation}
\frac{\partial \Psi_{\epsilon, q}(\sigma)}{\partial \sigma} < 0.
\end{equation}
We show this by an explicit computation.

\begin{align}
    \frac{\partial \Psi_{\epsilon, q}(\sigma)}{\partial \sigma} = &q \phi\left(-\sigma \log\left(\frac{h(q)}{q}\right) + \frac{1}{2\sigma}\right)\left(-\log\left(\frac{h(q)}{q}\right) - \frac{1}{2\sigma^2}\right) \\
    - &h(q)\phi\left(-\sigma \log\left(\frac{h(q)}{q}\right) - \frac{1}{2\sigma}\right)\left(- \log\left(\frac{h(q)}{q}\right) + \frac{1}{2\sigma^2}\right) 
\end{align}

Let's regroup the terms:
\begin{align}
    \frac{\partial \Psi_{\epsilon, q}(\sigma)}{\partial \sigma} = & - \frac{1}{2\sigma^2} \left[q \phi\left(-\sigma \log\left(\frac{h(q)}{q}\right) + \frac{1}{2\sigma}\right) + h(q)\phi\left(-\sigma \log\left(\frac{h(q)}{q}\right) - \frac{1}{2\sigma}\right)\right]\\
    & -\log\left(\frac{h(q)}{q}\right)\left[q \phi\left(-\sigma \log\left(\frac{h(q)}{q}\right) + \frac{1}{2\sigma}\right) - h(q)\phi\left(-\sigma \log\left(\frac{h(q)}{q}\right) - \frac{1}{2\sigma}\right)\right],
\end{align}
note that $h(q) = e^\epsilon - 1 + q > q$, therefore, the first term is negative. If we prove that the second term is also nonpositive, then we are finished. Let us prove that:

\begin{equation}
    q \phi\left(-\sigma \log\left(\frac{h(q)}{q}\right) + \frac{1}{2\sigma}\right) - h(q)\phi\left(-\sigma \log\left(\frac{h(q)}{q}\right) - \frac{1}{2\sigma}\right) = 0.
\end{equation}

Recall that $\phi(x) = \frac{1}{\sqrt{2\pi}}e^{-x^2/2}$ therefore this equality is equivalent to :

\begin{align}
    \exp\left(-\frac{1}{2}\left(-\sigma \log\left(\frac{h(q)}{q}\right) + \frac{1}{2\sigma}\right)^2 \right) &= \frac{h(q)}{q}\exp\left(-\frac{1}{2}\left(-\sigma \log\left(\frac{h(q)}{q}\right) - \frac{1}{2\sigma}\right)^2 \right)\\
    -\left(-\sigma \log\left(\frac{h(q)}{q}\right) + \frac{1}{2\sigma}\right)^2 &= 2\log\left(\frac{h(q)}{q}\right) -\left(-\sigma \log\left(\frac{h(q)}{q}\right) - \frac{1}{2\sigma}\right)^2\\ 
    2\sigma \log\left(\frac{h(q)}{q}\right) \frac{1}{\sigma} &= 2\log\left(\frac{h(q)}{q}\right),
\end{align}
which holds true. Therefore, we have proved that $\Psi_{\epsilon, q}(\sigma)$ is a decreasing function of $\sigma$. 
\end{proof}

Now, let us proceed with the conditions on $q$ and $\epsilon$. We are given that $\epsilon, q \ge 4\delta$. We will use a weaker condition $\frac{\min(e^\epsilon - 1, q)}{4} \ge \delta = \Psi_{\epsilon, q}(\sigma(q))$. This can be rewritten as

\begin{equation}
    \Psi^{-1}_{\epsilon, q}\left(\frac{\min(e^\epsilon - 1, q)}{4}\right) \le \sigma(q),
\end{equation}
since $\Psi_{\epsilon, q}$ is a strictly decreasing function. Our goal is to prove that $a - b < 0$, which is equivalent to:
\begin{equation}
    \frac{\sigma(q)}{\sqrt{2}}\log\left(\frac{e^{\epsilon} - 1 + q}{q}\right) > \frac{1}{2\sqrt{2}\sigma(q)} \Leftrightarrow \sigma(q) > \frac{1}{\sqrt{2\log\left(\frac{h(q)}{q}\right)}}
\end{equation}
We will achieve this by proving that

\begin{equation}
     \Psi_{\epsilon, q}^{-1}\left(\frac{\min(e^\epsilon - 1, q)}{4}\right) > \frac{1}{\sqrt{2\log\left(\frac{h(q)}{q}\right)}}.
\end{equation}
Since $\Psi$ is a decreasing function, we have:

\begin{equation}
     \frac{\min(e^\epsilon - 1, q)}{4} < \Psi_{\epsilon, q}\left(\frac{1}{\sqrt{2\log\left(\frac{h(q)}{q}\right)}}\right) = \frac{q}{2} - h(q)\Phi\left(-\sqrt{2\log\left(\frac{h(q)}{q}\right)}\right).
\end{equation}
Let us denote $\tau := \frac{e^\epsilon - 1}{q}$. Then $\frac{h(q)}{q} = \frac{e^\epsilon - 1 + q}{q} = \tau + 1$. Therefore,

\begin{equation}
    \frac{1}{2} - \frac{\min(\tau, 1)}{4} >  (\tau + 1) \Phi\left(-\sqrt{2\log\left(\tau + 1\right)}\right)
\end{equation}
Let $z = \sqrt{2\log\left(\tau + 1\right)}$. Then $\tau = e^{z^2/2} - 1$. We need to show that:

\begin{equation}
    \Phi(-z) < e^{-z^2/2}\left(\frac{1}{2} - \frac{\min(e^{z^2/2} - 1, 1)}{4}\right), \;\; \text{for} \;\; z > 0.
\end{equation}
We formulate  this purely technical statement as an auxiliary lemma, which will conclude the proof.

\begin{lemma}
\begin{equation}
    \Phi(-z) < e^{-z^2/2}\left(\frac{1}{2} - \frac{\min(e^{z^2/2} - 1, 1)}{4}\right), \;\; \text{for} \;\; z > 0.
\end{equation}
    
\end{lemma}

\begin{proof}

Consider the difference between these functions, denoted as $T(z)$:
\begin{equation}
    T(z) := \Phi(-z) - e^{-z^2/2}\left(\frac{1}{2} - \frac{\min(e^{z^2/2} - 1, 1)}{4}\right)
\end{equation}

Then the statement is equivalent to $T(z) < 0$ for all $z > 0$. For $z = 0$ we have $T(0) = 0$. Next, consider the interval  $0 \le z \le \sqrt{2\log\left(2\right)} \approx 1.18$. We can compute that $T(\sqrt{2\log\left(2\right)}) \approx -0.0055 < 0$. Now, consider the derivative of $T(z)$ within this range:

\begin{equation}
    T'(z) = -\phi(-z) + \frac{3z}{4}e^{-z^2/2} = e^{-z^2/2}\left(-\frac{1}{\sqrt{2\pi}} + \frac{3z}{4}\right) = \frac{3}{4}e^{-z^2/2}\left(z - \frac{4}{3\sqrt{2\pi}}\right).
\end{equation}
For $z <= \frac{4}{3\sqrt{2\pi}} \approx 0.53$, the function is strictly decreasing, reaching its minimal value at $z = \frac{4}{3\sqrt{2\pi}}$, with $T(\frac{4}{3\sqrt{2\pi}}) \approx -0.104$. The function $T(z)$ then monotonically increases towards the end of the interval while remaining negative. 

Now, consider the case when $z > \sqrt{2\log\left(2\right)}$:

\begin{equation}
    T'(z) = -\phi(-z) + \frac{z}{4}e^{-z^2/2} = e^{-z^2/2} \left(-\frac{1}{\sqrt{2\pi}} + \frac{z}{4}\right) = \frac{e^{-z^2/2}}{4} \left(z -\frac{4}{\sqrt{2\pi}}\right).
\end{equation}
The function $T(z)$ decreases until  $z = \frac{4}{\sqrt{2\pi}} \approx 1.60$, where it achieves its local minimum, $T(\frac{4}{\sqrt{2\pi}}) \approx -0.01$, and then monotonically increases as $z \to +\infty$, with $T(+\infty) = 0$. Therefore, $T(z) < 0$ for all $z >  0$, concluding the proof.
\end{proof}

\textbf{Remark.}

It is possible to improve the constant $4$ in the conjecture statement. Specifically, as long as $T(\sqrt{2\log2})$ remains negative, we can achieve the same result. One can show that for constants greater than $\frac{1}{\frac{1}{2} - 2\Phi(-\sqrt{2\log 2})} \approx 3.8319$, the inequality holds. Furthermore, this bound is tight, as we can numerically show the existence of $(\delta, q, \epsilon)$ such that $a > b$. For instance, when $\delta = 10^{-6}$, $q = 3.82 \cdot 10^{-6}$, and $\epsilon = 3.82 \cdot 10^{-6}$, we find that $\sigma \approx 0.8478$ and $a - b \approx 0.0015$.

\bibliographystyle{abbrvnat}
\bibliography{lit}

\begin{thebibliography}{1}
\providecommand{\natexlab}[1]{#1}
\providecommand{\url}[1]{\texttt{#1}}
\expandafter\ifx\csname urlstyle\endcsname\relax
  \providecommand{\doi}[1]{doi: #1}\else
  \providecommand{\doi}{doi: \begingroup \urlstyle{rm}\Url}\fi

\bibitem[Räisä et~al.(2024)Räisä, Jälkö, and Honkela]{subsampling}
O.~Räisä, J.~Jälkö, and A.~Honkela.
\newblock Subsampling is not magic: Why large batch sizes work for differentially private stochastic optimisation.
\newblock \emph{International Conference on Machine Learning (ICML)}, 2024.

\end{thebibliography}

\end{document}